\documentclass[11pt]{article}

\usepackage[letterpaper,margin=1in]{geometry}

\usepackage[T1]{fontenc}
\usepackage{lmodern}
\usepackage{microtype}
\usepackage{xcolor}

 \usepackage{relsize}
 \usepackage{bm}
 \usepackage{amsmath,amssymb,amsthm}
\usepackage{mathtools}
\usepackage{thmtools}

\usepackage{graphicx}
\usepackage{wrapfig}
\usepackage[hidelinks]{hyperref}
\usepackage[numbers,sort&compress]{natbib}

\declaretheorem[name=Theorem]{theorem}
\declaretheorem[name=Lemma]{lemma}

\declaretheorem[name=Corollary]{corollary}

\DeclareMathOperator{\Tr}{Tr}

\title{Symmetric Linear Dynamical Systems\\ are Learnable from Few Observations}
\author{Minh Vu, Andrey Y. Lokhov, and Marc Vuffray \\
\smaller \textit{Theoretical Division, Los Alamos National Laboratory, Los Alamos, NM 87545, USA}}
\date{}

\begin{document}
\maketitle

\begin{abstract}
We consider the problem of learning the parameters of a $N$-dimensional stochastic linear dynamics under both full and partial observations from a single trajectory of time $T$. We introduce and analyze a new estimator that achieves a small maximum element-wise error on the recovery of symmetric dynamic matrices using only $T=\mathcal{O}(\log N)$ observations, \textit{irrespective of whether the matrix is sparse or dense}. This estimator is based on the method of moments and does not rely on  problem-specific regularization. This is especially important for applications such as structure discovery.
\end{abstract}

\section{Introduction}
Learning parameters of stochastic linear dynamical systems \cite{aastrom1971system} is a fundamental problem across a broad range of disciplines, spanning time-series analysis, control theory, robotics, and modern statistical learning. Under this framework, at each time step $t$, the system is characterized by a state vector $x_t \in \mathbb{R}^N$ evolving according to the stochastic linear equation:
\begin{align}
    x_{t} &= Ax_{t-1} + \xi_{t-1},
    \label{eq:dynamics_formulation}
\end{align}
where $A \in \mathbb{R}^{N \times N}$ is a dynamic state matrix, and $\xi_t$ are i.i.d. zero-mean Gaussian noise vectors with covariance $\mathbb{E}[\xi_t\xi_t^\top] = \sigma^2  I$, where $I$ is the identity matrix of size $N \times N$. The objective is to estimate the elements of the dynamic state matrix $A$ from a trajectory of length $T$ either looking at the entire state vector $\{x_{t}\}_{t=0,\ldots,T}$ (in the case of full observations) or looking at a subset $x_{t,o} \in \mathbb{R}^{n_o}$ of entries of $x_{t}$ (in the case of partial observations).

This learning problem has a long history. Classical results in system identification  established asymptotic convergence guarantees for specific estimation methods such as maximum likelihood and least squares \cite{ljung1976consistency1, ljung1976consistency2, goodwin1977dynamic, ljung1987theory}. More recently, research focus has shifted toward non-asymptotic, finite-sample analyses that explicitly characterize the sample complexity of classical estimation methods such as least-squares, LASSO, or decomposition of Hankel matrix and new techniques under different settings. The problem statements include estimation in fully-observed \cite{faradonbeh2018finite, sarkar2019near,jedra2020finite,tyagi2023convex,Bento2010learning,nardi2011autoregressive,song2011large,kock2015oracle} and partially-observed \cite{oymak2021revisiting, simchowitz2018learning, sarkar2021finite, tsiamis2019finite, bakshi2023new} settings; estimation from the state vector only, like in Eq.~\eqref{eq:dynamics_formulation}, or from an additional known exogenous input sequence driving the system (a common setting in control theory) \cite{oymak2021revisiting,sarkar2021finite,tsiamis2019finite,bakshi2023new}; focusing on the role of the spectral radius assumptions in stable, marginally stable, and explosive systems \cite{faradonbeh2018finite, sarkar2019near,jedra2020finite,simchowitz2018learning, tsiamis2019finite, bakshi2023new}; or establishing sample-complexity lower bounds in these settings \cite{jedra2019lowerbound}. Most of prior works focused on providing finite sample-complexity guarantees for the $\ell_2$ norm of the dynamic state matrix.

In this paper, we focus on recovering the elements of the dynamic state matrix $A$ by providing guarantees on the maximum element-wise error. This setting is close to the one recently used in the literature on learning of graphical models \cite{klivans2017learning, vuffray2020efficient}. It has a straightforward application in recovering the support of the matrix $A$. It could also be useful in real applications such as anomaly detection \cite{hannon2021real} where the detection criterion is based on the individual elements of the matrix instead of the aggregate norm. We consider a simple variant of the Eq.~\eqref{eq:dynamics_formulation}, where we assume that the matrix $A$ is symmetric and stable with the spectral radius $\rho(A) < 1$. We introduce a new estimator based on the methods of moments, and show that in the setting of full observations, it recovers $A$ to a finite element-wise max norm error using only the trajectory of length $T = \mathcal{O}(\log N)$. In this regime, the total number of data points $\mathcal{O}(N \log N)$ is thus smaller than the number of unknown parameters $N(N-1)/2$. In addition, our approach can be readily used in the challenging setting of partial observations, where we show that the sub-matrix of $A$ corresponding to observed nodes can also be learned to a finite element-wise max norm error using $T=\mathcal{O}(\log N)$.

A possibility of support recovery from a time trajectory with length scaling as $T = \mathcal{O}(\log N)$ has been previously shown from the analysis of $\ell_1$-regularized least-squares in the sparse setting and under assumptions of incoherence and stronger stability in \cite{Bento2010learning}. Surprisingly, we show that our new estimator does not require a problem-specific regularization or extra assumptions, and \emph{succeeds in both sparse and dense case with $\mathcal{O}(\log N)$ observations}. In this work, we also illustrate the scaling behavior of our algorithm with specific numerical examples of sparse and dense matrices $A$.

\section{Related Work}

Identifying linear time-invariant systems from data has a decades-long history in time-series analysis and system identification (see \cite{ljung1987theory,van2012subspace,galrinho2016least} and references therein). Classical results primarily provide asymptotic convergence guarantees for learning models from data. Some of the earlier non-asymptotic studies in system identification include \cite{weyer2002finite,campi2002finite,vidyasagar2008learning}. However, these results are often quite conservative, featuring bounds that scale exponentially with the system size.

Recently, the machine learning community has shown significant interest in deriving sharp non-asymptotic error bounds for linear system identification. The computationally simple ordinary least squares (OLS) estimator has been the main focus of these analyses. While OLS is straightforward to implement, it is more difficult to analyze, since standard analyses for OLS on random design linear regression \cite{hsu2014random} cannot be directly applied due to correlations between samples and the noise process. Analysis of OLS has been performed using mixing-time arguments (see, e.g., \cite{yu1994rates}) in \cite{faradonbeh2018finite} and based on Mendelson’s small-ball method in \cite{simchowitz2018learning}, extended in \cite{sarkar2019near} for the unstable regime.

Improved sample complexity can be achieved in the sparse case, where $\ell_1$-regularized least squares (LASSO-type) estimators are widely used. The work \cite{Bento2010learning} analyzed support recovery for high-dimensional systems, and \cite{nardi2011autoregressive} provided consistency and sample complexity guarantees under moderate-dimensional regimes. In parallel, related studies \cite{song2011large, kock2015oracle} established oracle inequalities and support-selection consistency for high-dimensional autoregressive models. Extensions of this line of work include systems with control inputs \cite{fattahi2019learning} and time-delayed dynamics \cite{han2015direct}.

The setting of partially observed systems presents substantially greater difficulties, and hence finite-sample analyses for  systems has appeared only recently \cite{oymak2021revisiting,simchowitz2019learning,sarkar2021finite,tsiamis2019finite,bakshi2023new}, leading to the required sample-complexity scaling as high-degree polynomials in the system size $N$. In this setting, the system matrices are typically identifiable only up to a similarity transformation as all such transformations lead to an equivalent dynamics. \cite{oymak2021revisiting} revisited the classical Kalman–Ho subspace method, which estimates Markov parameters and then reconstructs the system matrices. \cite{simchowitz2019learning} proposed a pre-filtering step that stabilizes estimation for marginally stable systems while maintaining the same complexity. \cite{bakshi2023new} introduced a convex-optimization method-of-moments-based estimator to improve Markov parameter estimation with minimal assumptions. When the model order is unknown, \cite{sarkar2021finite} developed an adaptive approach that learns the system order directly from data. For systems without control inputs and driven purely by stochastic noise, \cite{tsiamis2019finite} analyzed subspace identification under Kalman filter convergence.

\section{Methods and Results}

\subsection{Learning  fully observed linear dynamical systems}
We start with an $N$-dimensional linear dynamical system governed by
\begin{align}
\label{eq: model}
x_t = A x_{t-1} + \xi_{t-1},
\end{align}
where $A \in \mathbb{R}^{N \times N}$ is a symmetric matrix  satisfying $\rho(A) < 1$, $x_t \in \mathbb{R}^N$ denotes the system state, and $\xi_t \in \mathbb{R}^n$ are i.i.d. zero-mean Gaussian noise vectors with covariance $\mathbb{E}[\xi_t\xi_t^\top] = \sigma^2 I$. The initial state is assumed to be $x_0 = 0$.
Iterating the dynamic, we obtain the solution for $x_t$, i.e.,
\begin{align}
\label{eq4}
x_t = \sum_{u=0}^{t-1} A^{t-1-u} \xi_u.
\end{align}
From this relation, we observer that the time-shifted covariance between states $x_t$ and $x_s$, where $s \geq t$, is expressed as
\begin{align}
\label{eq5}
\mathbb{E}[x_t x_s^\top]
= A^{s-t}  \sigma^2 \sum_{u=0}^{t-1} A^{2(t-1-u)}
= A^{s-t}  \mathbb{E}[x_t x_t^\top].
\end{align}
This relation shows that the temporal correlations encode powers of the dynamics matrix $A$. Motivated by this observation, we propose the following family of  estimators that recovers powers of the dynamic matrix directly from trajectory data:
\begin{align}
\label{eq: estimator}
\hat{\mathbb{S}}_{m} (T) := \frac{1}{T-m}\sum_{t=0}^{T-m-1}{x}_t{x}_{t+m}^\top - \frac{1}{T-m-2}\sum_{t=0}^{T-m-3}{x}_t{x}_{t+m+2}^\top.
\end{align}
The next theorem establishes that this estimator is asymptotically unbiased, showing that the expected value of $\hat{\mathbb{S}}_m (T)$ converges to $\sigma^2{A}^m$ as $T\rightarrow\infty$.

\begin{theorem}
\label{Theorem: expectation}
\textbf{(Asymptotic Analysis):} Consider a dynamical system defined in Eq.~\eqref{eq: model}, where $A \in \mathbb{R}^{N\times N}$ is a symmetric matrix with spectral radius $\rho(A) < 1$,
and with initial state $x_0 = 0$.
Suppose ${\xi_t}$ are i.i.d. Gaussian random variables with $\mathbb{E}[\xi_t] = 0$ and $\mathbb{E}[\xi_t \xi_t^\top] = \sigma^2 I$. Then, the expected value of the estimators $\hat{\mathbb{S}}_m(T)$ defined in \eqref{eq: estimator} are given by
$$
    \mathbb{E}[\hat{\mathbb{S}}_m(T)] = \sigma^2 A^m + h_m(T),
$$
where the bias term is equal to $h_m(T)=\frac{1}{T-m}  \sigma^2A^m(A^{2(T-m)}-I)(I-A^2)^{-2} - \frac{1}{T-m-2} \sigma^2A^{m+2}(A^{2(T-m-2)}-I)(I-A^2)^{-2}$, implying that $h_m(T)$ decays as $\mathcal{O}(\frac{1}{T})$ as $T \to \infty$.
\end{theorem}

Our main theoretical result shows that, with a logarithmic dependence on the system dimension, our estimators recover the parameters of powers of the dynamical matrix uniformly.

\begin{theorem}
    \label{Theorem: main}
    \textbf{(Finite-Sample Analysis):}  Consider a dynamical system defined in Eq.~\eqref{eq: model}, where $A \in \mathbb{R}^{N\times N}$ is a symmetric matrix with spectral radius $\rho(A) < 1$,
and with initial state $x_0 = 0$.
Suppose ${\xi_t}$ are i.i.d. Gaussian random variables with $\mathbb{E}[\xi_t] = 0$ and $\mathbb{E}[\xi_t \xi_t^\top] = \sigma^2 I$.
   Let the estimator $\hat{\mathbb{S}}_m(T)$ be defined as in \eqref{eq: estimator}. For any $0<\epsilon <\frac{4\sigma^2}{(1-\rho(A))^2}$ and $\delta >0$, if 
    \begin{align*}
        T \ge \max\left(\frac{64\sigma^4\Big(1+2\sqrt{\log(2N^2/\delta)}\Big)^2}{\epsilon^2(1-\rho(A))^4},  \,\,  2(m+2) \right),
    \end{align*}
    then, with probability of at least $1-\delta$, $\|\hat{\mathbb{S}}_m(T) - A^m\sigma^2\|_{\max} \le \epsilon$. 
\end{theorem}
This result implies that both $A$ and its higher-order powers can be estimated with a fixed precision in element-wise maximum norm, i.e., $\|A\|_{\max}:=\max_{i,j}|a_{i,j}|$, using only $\mathcal{O}(\log N)$ samples. Importantly, these higher-order moment estimates also enable a natural extension of our estimator to the partially observed setting, as we will see immediately.

\subsection{Learning partially observed linear dynamical systems}
We consider systems with both observed and hidden components. Let the state vector $x_t \in \mathbb{R}^N$ be partitioned into its observed component $x_{t,o} \in \mathbb{R}^{n_o}$ and its hidden (or latent) component $x_{t,h} \in \mathbb{R}^{n_h}$, such that $x_t = [x_{t,o}^\top, x_{t,h}^\top]^\top$ and $n = n_o + n_h$. This partitioning imposes a corresponding block structure on the system dynamics described in \eqref{eq: model},
\begin{align}
    \label{eq: model (partial obs)}
    \begin{bmatrix}
        x_{t,o} \\ x_{t,h}
    \end{bmatrix}
    =
    \underbrace{\begin{bmatrix}
        B & C \\ C^\top & E
    \end{bmatrix}}_{A}
    \begin{bmatrix}
        x_{t-1,o} \\ x_{t-1,h}
    \end{bmatrix} + 
    \begin{bmatrix}
        \xi_{t-1,o} \\ \xi_{t-1,h}
    \end{bmatrix}, 
\end{align}
where $B \in \mathbb{R}^{n_o \times n_o}$ and $E \in \mathbb{R}^{n_h \times n_h}$ represent the internal dynamics of the observed and hidden subsystems, respectively, while $C \in \mathbb{R}^{n_o \times n_h}$ and $C^\top$ describe the coupling between them.

Since only $\{x_{t,o}\}_{t=0}^{T-1}$ is observable, we could still compute the submatrix of $\hat{\mathbb{S}}_m(T)$ restricted to the observed coordinates:
\begin{align}
    \label{eq:estimator (partial obs)}
    [\hat{\mathbb{S}}_m(T)]_{\{i,j\}\in \mathcal{O}} = \frac{1}{T-m}\sum_{t=0}^{T-m-1}{x}_{t,o}{x}_{t+m,o}^\top - \frac{1}{T-m-2}\sum_{t=0}^{T-m-3}{x}_{t,o}{x}_{t+m+2,o}^\top
\end{align}
where $\mathcal{O}=\{1,\dots,n_o\}$ denotes the set of observed coordinates. By Theorem~\ref{Theorem: main}, $[\hat{\mathbb{S}}_m(T)]_{\{i,j\}\in \mathcal{O}}$ estimates $[A^m\sigma^2]_{\{i,j\}\in \mathcal{O}}$ with accuracy $\epsilon$. For small $m$, we obtain explicitly:
\begin{align*}
    &\|[\hat{\mathbb{S}}_0(T)]_{\{i,j\}\in \mathcal{O}} - I\sigma^2\|_{\max} \le \epsilon \\
    &\|[\hat{\mathbb{S}}_1(T)]_{\{i,j\}\in \mathcal{O}}- B\sigma^2\|_{\max} \le \epsilon \\
    &\|[\hat{\mathbb{S}}_2(T)]_{\{i,j\}\in \mathcal{O}} - (B^2+CC^\top)\sigma^2\|_{\max} \le \epsilon \\
    &\|[\hat{\mathbb{S}}_3(T)]_{\{i,j\}\in \mathcal{O}} - (B^3+CC^\top B+BCC^\top+CEC^\top)\sigma^2\|_{\max} \le \epsilon. 
\end{align*}
These relations reveal that sub-matrices $\{B, C, E\}$ can be reconstructed successively from low-order moment estimates.

\begin{corollary}
\label{corollary1}
\textbf{(Recovery under Partial Observations):} Consider problem of learning the system dynamics from an observed trajectory $\{x_{t,o}\}_{t=0,\dots,T}$, generated by the model \eqref{eq: model (partial obs)}. Let $[\hat{\mathbb{S}}_m(T)]_{\{i,j\}\in \mathcal{O}}$ be calculated as in \eqref{eq:estimator (partial obs)} and define the following block estimators:
\begin{itemize}
    \item $\hat{\sigma}^2 \coloneqq \frac{1}{n_o}\mathrm{Tr}\big([\hat{\mathbb{S}}_0(T)]_{\{i,j\}\in \mathcal{O}} \big)$
    \item $\hat{B} \coloneqq \frac{1}{\hat{\sigma}^2} [\hat{\mathbb{S}}_1(T)]_{\{i,j\}\in \mathcal{O}} $
    \item $\widehat{CC^\top} \coloneqq \frac{1}{\hat{\sigma}^2} [\hat{\mathbb{S}}_2(T)]_{\{i,j\}\in \mathcal{O}}  - \hat{B}^2$
    \item $\widehat{CEC^\top} \coloneqq \frac{1}{\hat{\sigma}^2} [\hat{\mathbb{S}}_3(T)]_{\{i,j\}\in \mathcal{O}}  - \hat{B}^3 - \widehat{CC^\top}\hat{B} - \hat{B}\widehat{CC^\top}$
\end{itemize}

Let $\kappa = \max\{64\sigma^2,32^2\}$, $0<\epsilon <\frac{\sigma^2}{2}$, and $\delta >0$. Then, with probability of at least $1-\delta$,
\begin{itemize} 
    \item If $T \ge \max\left(\frac{\kappa\Big(1+2\sqrt{\log(2N^2/\delta)}\Big)^2}{\epsilon^2(1-\rho(A))^4},  \,\,  6 \right)$, then $|\hat{\sigma}^2 - \sigma^2| \le \epsilon$ and $\|\hat{B} - B\|_{\max} \le \epsilon$. 
    \item If $T \ge \max\left(\frac{20^2\kappa N^2\Big(1+2\sqrt{\log(2N^2/\delta)}\Big)^2}{\epsilon^2(1-\rho(A))^4}, \, 8 \right)$, then $\|\widehat{CC^\top} - CC^\top\|_{\max} \le \epsilon$.
    \item If $T \ge \max\left(\frac{140^2\kappa N^4\Big(1+2\sqrt{\log(2N^2/\delta)}\Big)^2}{\epsilon^2(1-\rho(A))^4}, \, 10 \right)$, then $\|\widehat{CEC^\top} - CEC^\top\|_{\max} \le \epsilon$.
\end{itemize}
\end{corollary}

Corollary~\ref{corollary1} shows that the observed subsystem $B$ of the system dynamics can be recovered using only $\mathcal{O}(\log N)$ samples, while accurate estimation of $CC^\top$ and $CEC^\top$ requires $\mathcal{O}(N^2\log N)$ and $\mathcal{O}(N^4\log N)$ samples, respectively. Note that the submatrices $C$ and $E$ themselves can be recovered only up to a linear transformation from $CC^\top$ and $CEC^\top$. This is an unavoidable consequence of an intrinsic symmetry of the recovery problem with partial observations. Indeed, if $O$ is an arbitrary orthogonal matrix, then the distribution of $\{x_{t,o}\}_{t=0,\dots,T}$ is invariant under the transformation $\tilde{C}=CO$, $\tilde{E}=O^\top EO$. Note that the distribution of $\{x_{t,h}\}_{t=0,\dots,T}$ is mapped to the distribution of $\tilde{x}_{t,h}= O^{\top}x_{t,h}$, but since $x_{t,h}$ is unobserved, $\tilde{C}$ and $\tilde{E}$ are undistinguishable from $C$ and $E$. This undistinguishability manifests itself in our estimators through $CC^\top = \tilde{C}\tilde{C}^\top$ and $CEC^\top = \tilde{C}\tilde{E}\tilde{C}^\top$.



\section{Analysis}
\label{sec:analysis}

This section provides the proofs of the main results. We first show that the proposed estimator is asymptotically unbiased (Theorem~\ref{Theorem: expectation}). Then, we prove a number of auxiliary lemmas and establish a finite-sample, elementwise concentration bound of the estimator (Lemma~\ref{lem:Chernoff elementwise}), which leads to Theorem~\ref{Theorem: main}. Finally, we give the proof of Corollary~\ref{corollary1} for the partially observed setting.

\medskip
\noindent\textit{Notation.} Throughout this section, $\|\cdot\|_{\max}$ denotes the elementwise maximum norm, $\|\cdot\|_2$ the spectral norm, $\rho(\cdot)$ the spectral radius, and $\lambda_i(M)$ the eigenvalues of matrix $M$. Our analyses will frequently use the definition of the linear system in \eqref{eq: model} and of the estimator in \eqref{eq: estimator} defined in the Methods section.

\subsection{Proof of Theorem~\ref{Theorem: expectation}}
\begin{proof}
By iterating \eqref{eq: model}, we have
\(
x_t = \sum_{u=0}^{t-1} A^{t-1-u}\xi_u
\)
For $s>t$,
\begin{align*}
\mathbb{E}[x_tx_s^\top] &= \sum_{u=0}^{t-1}\sum_{v=0}^{s-1} A^{t-1-u} \mathbb{E}[\xi_{u}\xi_{v}^\top] A^{s-1-v}  \\
&= \sigma^2\sum_{u=0}^{t-1} A^{t-1-u}A^{s-1-u} = \sigma^2 A^{s-t} \sum_{u=0}^{t-1} A^{2(t-1-u)} =  \sigma^2 A^{s-t}  \sum_{t'=0}^{t-1} A^{2t'}.
\end{align*}
Then, using the eigendecomposition of $A=UDU^\top$, we have
\begin{align*}
      \mathbb{E}[x_tx_s^\top] = \sigma^2  A^{s-t}  U^{-1} \big(\sum_{t'=0}^{t-1} D^{2t'} \Big)U  = \sigma^2 A^{s-t}  U^{-1} (I-D^{2t})(I-D^2)^{-1} U =  \sigma^2 A^{s-t} (I-A^{2t})(I-A^2)^{-1}.
\end{align*}
In forming $(I-D^2)^{-1}$, we used $\rho(A)<1$. Substituting it into \eqref{eq: estimator} and collecting terms yields
\[
\mathbb{E}[\hat{\mathbb{S}}_m(T)]
= \sigma^2 A^m
+ \frac{1}{T-m}A^m\sigma^2(A^{2(T-m)}-I)(I-A^2)^{-2}
- \frac{1}{T-m-2}A^{m+2}\sigma^2(A^{2(T-m-2)}-I)(I-A^2)^{-2}.
\]
Since $\rho(A)<1$, the bias term is $h_m(T)=\mathcal{O}(1/T)$, proving the claim.
\end{proof}

\subsection{Auxiliary Lemmas}

\begin{lemma}[Quadratic Representation]
\label{lem:quadratic_form}
Let $X=[x_1^\top,\dots,x_{T-1}^\top]^\top$. For any $i,j$,
\(
[\hat{\mathbb{S}}_m(T)]_{ij}=X^\top G_{ij}X,
\)
where $G_{ij}\in\mathbb{R}^{n(T-1)\times n(T-1)}$ is a symmetric block-banded matrix, with nonzero blocks corresponding to lags $m$
and $m+2$ in \eqref{eq: estimator}, and entries given by $c_1=\frac{1}{2(T-m)}$ and $c_2=\frac{1}{2(T-m-2)}$.
\end{lemma}
\begin{proof}
A direct computation yields $[\hat{\mathbb{S}}_m(T)]_{ij} = \frac{1}{T-m}\sum_{t=0}^{T-m-1}x_{t,i}x_{t+m,j} - \frac{1}{T-m-2}\sum_{t=0}^{T-m-3}{x}_{t,i}{x}_{t+m+2,j}$
\setcounter{MaxMatrixCols}{8}
\begin{align*}
    = \underbrace{\begin{bmatrix} x_1^\top & x_2^\top & \cdots & x_{T-1}^\top \end{bmatrix}}_{\mathlarger{X^\top}}
    \setlength\arraycolsep{3pt}
    \underbrace{
    \begingroup\renewcommand*{\arraystretch}{1.7}
    \begin{bmatrix} \begin{array}{c|c|c|c|c|c|c|c} 
    \bm{0} & \cdots & e_1^{ij} & \bm{0} & e_2^{ij} & \bm{0} & \cdots & \bm{0} \\ \hline \vdots & \ddots & & \ddots & & \ddots & & \vdots \\ \hline e_1^{ij} & & \ddots & & \ddots & & \ddots & \bm{0} \\ \hline \bm{0} & \ddots & & \ddots & & \ddots & & e_2^{ij} \\ \hline e_2^{ij} & & \ddots & & \ddots & & \ddots & \bm{0} \\ \hline \bm{0} & \ddots & & \ddots & & \ddots & & e_1^{ij} \\ \hline \vdots & & \ddots & & \ddots & & \ddots & \vdots \\ \hline \bm{0} & \cdots & \bm{0} & e_2^{ij} & \bm{0} & e_1^{ij} & \cdots & \bm{0} 
    \end{array} 
    \end{bmatrix}
    \endgroup}_{\mathlarger{G_{ij}}}  
    \underbrace{\begin{bmatrix}
       x_{1} \\x_{2}\\  \vdots \\ x_{T-1}  
    \end{bmatrix} }_{\mathlarger{X}},
\end{align*}
where $e_1^{ij}$ and $e_2^{ij}$
denote $N\times N$ matrices that are zero everywhere except for the entry at position $(i,j)$, where they have elements $\frac{1}{2(T-m)}$ and $\frac{1}{2(T-m-2)}$, respectively. \\
\end{proof}

\begin{lemma}[Trajectory Distribution]
\label{lem:traj_distribution}
Let $X=[x_1^\top,\dots,x_{T-1}^\top]^\top$, then
\(
X\sim \mathcal{N}(0,\ \sigma^2 L L^\top),
\)
where $L$ is a block lower-triangular matrix with powers of $A$ on its sub-diagonals.
\end{lemma}

\begin{proof}
\noindent From $x_{k+1} = A x_k + \xi_k$, we have 
\begin{align*}
    \underbrace{\begin{bmatrix}
        x_1 \\ x_2 \\  \vdots \\ x_{T-1} 
    \end{bmatrix} }_{\mathlarger{X}} =
    \underbrace{\begin{bmatrix}  
    I \\    
    A & I \\
    A^2 & A & I \\
    \vdots & &  \ddots & \ddots\\
    A^{T-2} & \cdots & \cdots & A & I
    \end{bmatrix}}_{\mathlarger{L}}
    \underbrace{\begin{bmatrix}
        \xi_0 \\ \xi_1 \\  \vdots \\ \xi_{T-2} 
    \end{bmatrix}. }_{\mathlarger{\Xi}}
\end{align*}
Since $\Xi \sim \mathcal{N}(0,\, \sigma^2 I)$, we have $X \sim \mathcal{N}(0,\, \sigma^2 LL^\top )$. \\
\end{proof}

\begin{lemma}[Spectral Radius Bound]
\label{lem:spectral_bound}
Let $\bar{\lambda} := \big(\frac{1}{T-m}+\frac{1}{T-m-2}\big)\frac{1}{(1-\rho(A))^2}$. Then
\(
\rho(L^\top G_{ij}L)\le \bar{\lambda}.
\)
\end{lemma}

\begin{proof}
\noindent By sub-multiplicativity of the spectral norm, we have
\begin{align*}
    \rho(L^\top G_{ij} L) = \max_i |\lambda_i(L^\top G_{ij} L)|   = \sigma_{\max} (L^\top G_{ij} L) \le \sigma_{\max}(G_{ij}) \sigma_{\max}^2(L) = \sigma_{\max}(G_{ij}) \lambda_{\max}(LL^\top). 
\end{align*}
From the block-banded structure of $G_{ij}$ and Gershgorin circle theorem,
\begin{align*}
    \sigma_{\max}(G_{ij}) = \max_{i} |\lambda_i(G_{ij})| \le \frac{1}{T-m} + \frac{1}{T-m-2}.
\end{align*}
To upper bound $\lambda_{\max}(LL^\top)$, we consider the lower bound of $\lambda_{\min}\big((LL^\top)^{-1}\big)$. By rotating $(LL^\top)^{-1}$, 
\begin{align*}
\setlength{\arraycolsep}{5pt}
    \begin{bmatrix}
        U & & \\
         & \ddots & \\
         & & U
    \end{bmatrix} 
    \underbrace{\begin{bmatrix}
       I{+}A^2 & -A &    0    \\
       -A & \ddots & \ddots & \ddots \\
          0     & \ddots & \ddots & \ddots & 0  \\
                  & \ddots & \ddots & I{+}A^2  & -A\\
                &        &     0     &-A & I 
    \end{bmatrix}}_{(LL^\top)^{-1}}  
    \begin{bmatrix}
        U^\top & & \\
         & \ddots & \\
         & & U^\top
    \end{bmatrix}  =  \underbrace{\begin{bmatrix}
        I{+}D^2 & -D &    0    \\
       -D & \ddots & \ddots & \ddots \\
          0     & \ddots & \ddots & \ddots & 0  \\
                  & \ddots & \ddots & I{+}D^2  & -D\\
                &        &     0     & -D & I
    \end{bmatrix}}_{\mathcal{M}},
\end{align*}
\noindent and applying Gershgorin circle theorem to $\mathcal{M}$, we have that each eigenvalue of $\mathcal{M}$ is lower bounded by $(1-|\lambda_i(A)|)^2$. Thus, $\lambda_{\min}\big((LL^\top)^{-1}\big) \ge (1-\rho(A))^2$ and hence $\lambda_{\max}(LL^\top) \le  \frac{1}{(1-\rho(A))^2}$. 

Combining the bounds yields, $\rho(L^\top G_{ij} L) \le \Big( \frac{1}{T-m} + \frac{1}{T-m-2} \Big) \frac{1}{(1-\rho(A))^2} =: \bar{\lambda}$. \qedhere
\end{proof}

\begin{lemma}[Log-Determinant Trace Identity]
\label{lemma:Log-determinant trace}
If $H$ is diagonalizable with $|\lambda_i(H)|<1$, then
\(
-\log\det(I-H)=\mathrm{Tr}\!\left(\sum_{k=1}^\infty \frac{H^k}{k}\right).
\)
\end{lemma}

\begin{proof}
\noindent For $|x|<1$, the Maclaurin series of $\log(1-x) = -\sum_{k=1}^\infty \frac{x^k}{k}$ (see example 6.4.4 in \cite{bilodeau2010introduction}).

Since $H$ is diagonalizable, $H = P\Lambda P^{-1}$ with $\Lambda = \mathrm{diag}(\lambda_1,\dots,\lambda_n)$. Then
\[
\sum_{k=1}^\infty \frac{H^k}{k}
= P\!\left(\sum_{k=1}^\infty \frac{\Lambda^k}{k}\right)P^{-1}
= -P\,\mathrm{diag}\!\big(\log(1-\lambda_1),\dots,\log(1-\lambda_n)\big)P^{-1}.
\]
Taking traces and using $\Tr(PAP^{-1})=\Tr(A)$ gives
\[
\Tr\!\left(\sum_{k=1}^\infty \frac{H^k}{k}\right)
= -\sum_{i=1}^n \log(1-\lambda_i)
= -\log\!\Big(\prod_{i=1}^n (1-\lambda_i)\Big)
= -\log\det(I-H). \qedhere
\]
\end{proof}

\begin{lemma}[Concentration of Estimator Elements]
\label{lem:Chernoff elementwise}
For any $0<\epsilon <\frac{4\sigma^2}{(1-\rho(A))^2}$, the probability of an element $[\hat{\mathbb{S}}_m(T)]_{ij}$ deviating from its expectation by more than $\epsilon$ is bounded by:
$$
P\left(\left|[\hat{\mathbb{S}}_m(T)]_{ij} - \mathbb{E}[\hat{\mathbb{S}}_m(T)]_{ij}\right| \ge \epsilon\right) \le 2\exp \left(- \frac{\epsilon^2}{16T\bar{\lambda}^2\sigma^4} \right)
$$
where $\bar{\lambda} = \left( \frac{1}{T-m} + \frac{1}{T-m-2} \right) \frac{1}{(1-\rho(A))^2}$. 
\end{lemma}

\begin{proof}
\noindent
Since $\rho(L^\top G_{ij}L)\le\bar{\lambda}$ (from Lemma~\ref{lem:spectral_bound}), for any $|u|<\tfrac{1}{2\sigma^2\bar{\lambda}}$, we have $|\lambda_i(2u\sigma^2L^\top G_{ij}L)|<1.$
This condition ensures that the moment generating function
$\mathbb{E}[e^{uX^\top G_{ij}X}]$ exists and can be evaluated, 
\[
\mathbb{E}[e^{uX^\top G_{ij} X}]
= \big(\det(I-2u\sigma^2L^\top G_{ij}L)\big)^{-1/2}
= \exp\!\Big(-\frac{1}{2}\log\det(I-2u\sigma^2L^\top G_{ij}L)\Big).
\]
Applying Lemma~\ref{lemma:Log-determinant trace},
\[
\mathbb{E}[e^{uX^\top G_{ij} X}]
= \exp\!\Big(\frac{1}{2}\mathrm{Tr}\!\sum_{k=1}^\infty
\frac{(2u\sigma^2L^\top G_{ij}L)^k}{k}\Big)=\exp\Big(\frac{1}{2} \Tr\Big(\sum_{k=2}^{\infty} \frac{(2u\sigma^2L^\top G_{ij}L)^k}{k}\Big) + u \sigma^2 \Tr(L^\top GL) \Big) .
\]

Since
$\mathbb{E}[X^\top G_{ij}X]=\sigma^2\mathrm{Tr}(L^\top G_{ij}L)$
and $\mathrm{Tr}\big((L^\top G_{ij}L)^k\big)\le 2T\bar{\lambda}^k$ (due to $\mathrm{rank}(L^\top G_{ij}L)\le 2T$),
\begin{align*}
\mathbb{E}\big[e^{uX^\top G_{ij} X}\big]e^{-u\mathbb{E}[X^\top G_{ij} X]} 
\le \exp\!\Big(\frac{1}{2}\sum_{k=2}^\infty
\frac{(2|u|\sigma^2)^k(2T\bar{\lambda}^k)}{k}\Big)
= \exp\!\Big(-T\log(1-2|u|\bar{\lambda}\sigma^2)-2T|u|\bar{\lambda}\sigma^2\Big).
\end{align*}

Then, by Chernoff's method, we have
\[
P\!\left(\big|[\hat{\mathbb{S}}_m(T)]_{ij}-\mathbb{E}[\hat{\mathbb{S}}_m(T)]_{ij}\big|\ge\epsilon\right)
\le
2\exp\Big( \min_{0<u<\frac{1}{2\bar{\lambda}\sigma^2}}  -T \log(1-2u\bar{\lambda}\sigma^2) - 2Tu\bar{\lambda}\sigma^2  - u\epsilon \Big).
\]
To optimize over $u$, we define
$g(u)=-T\log(1-2u\bar{\lambda}\sigma^2)-2Tu\bar{\lambda}\sigma^2-u\epsilon$. Note that $g(u)$ is strictly convex on $(0,\tfrac{1}{2\bar{\lambda}\sigma^2})$.
Setting $g'(u^*)=0$ gives
\[
u^*=\frac{\epsilon}{2\bar{\lambda}\sigma^2(2T\bar{\lambda}\sigma^2+\epsilon)}.
\]
Substituting $u^*$ yields
\[
P\!\left(\big|[\hat{\mathbb{S}}_m(T)]_{ij}-\mathbb{E}[\hat{\mathbb{S}}_m(T)]_{ij}\big|\ge\epsilon\right)
\le
2\exp\!\Big(
T\big(\log(1+\tfrac{\epsilon}{2T\bar{\lambda}\sigma^2})
-\tfrac{\epsilon}{2T\bar{\lambda}\sigma^2}\big)
\Big).
\]
Finally, using $\log(1+z)-z\le -\tfrac{z^2}{4}$ for $0<z<1$
and noting $\tfrac{\epsilon}{2T\bar{\lambda}\sigma^2}<1$
since $\epsilon<\tfrac{4\sigma^2}{(1-\rho(A))^2}$,
we simplify the expression and obtain
\[
P\!\left(\big|[\hat{\mathbb{S}}_m(T)]_{ij}-\mathbb{E}[\hat{\mathbb{S}}_m(T)]_{ij}\big|\ge\epsilon\right)
\le
2\exp\!\left(-\frac{\epsilon^2}{16T\bar{\lambda}^2\sigma^4}\right).
\qedhere
\]
\end{proof}

\subsection{Proof of Theorem~\ref{Theorem: main}}
\begin{proof}[Proof of Theorem~\ref{Theorem: main}]
By Lemma~\ref{lem:Chernoff elementwise} and a union bound over all pairs $(i,j)$,
\[
P\!\left(\|\hat{\mathbb{S}}_m(T)-\mathbb{E}[\hat{\mathbb{S}}_m(T)]\|_{\max}\ge\epsilon\right)
\le
\sum_{i,j} P\!\left(|[\hat{\mathbb{S}}_m(T)]_{ij}-\mathbb{E}[\hat{\mathbb{S}}_m(T)]_{ij}|\ge\epsilon\right)
\le
2N^2\exp\!\left(-\frac{\epsilon^2}{16T\bar{\lambda}^2\sigma^4}\right).
\]
Hence, with probability at least $1-\delta$,
\[
\|\hat{\mathbb{S}}_m(T)-\mathbb{E}[\hat{\mathbb{S}}_m(T)]\|_{\max}
\le
\sqrt{
16T\sigma^4
\!\left(\frac{1}{T-m}+\frac{1}{T-m-2}\right)^{\!2}
\frac{\log(2N^2/\delta)}{(1-\rho(A))^4}
}
\le
16\sigma^2\sqrt{\frac{\log(2N^2/\delta)}{T(1-\rho(A))^4}}.
\]
The last inequality is due to 
$\frac{1}{T-m}+\frac{1}{T-m-2}\le\frac{4}{T}$, for $T\ge2(m+2)$.

From Theorem~\ref{Theorem: expectation},
$\mathbb{E}[\hat{\mathbb{S}}_m(T)]=A^m\sigma^2+h(A,T)$,
where $\|h(A,T)\|_{\max}\le\frac{8\sigma^2}{T(1-\rho^2(A))^2}$ for $T\ge2(m+2)$.
By the triangle inequality, we have
\[
\|\hat{\mathbb{S}}_m(T)-A^m\sigma^2\|_{\max}
\le
16\sigma^2\sqrt{\frac{\log(2N^2/\delta)}{T(1-\rho(A))^4}}
+\frac{8\sigma^2}{T(1-\rho^2(A))^2} \le \frac{8\sigma^2(1+2\sqrt{\log(2N^2/\delta)})}{(1-\rho(A))^2\sqrt{T}}.
\]
Requiring the right-hand side to be at most $\epsilon$
gives the stated lower bounds on $T$.
Therefore, under the displayed condition on $T$, $\|\hat{\mathbb{S}}_m(T)-A^m\sigma^2\|_{\max}\le\epsilon,
$ with probability at least $1-\delta$.
\qedhere
\end{proof}

\subsection{Proof of Corollary~\ref{corollary1}}

\begin{proof}[Proof of Corollary~\ref{corollary1}]
Let $\hat{S}_m \coloneqq [\hat{\mathbb{S}}_m(T)]_{\{i,j\}\in\mathcal{O}}$. Note that $0<\epsilon <\frac{\sigma^2}{2}$ implies $0<\epsilon <\frac{4\sigma^2}{(1-\rho(A))^2}$ since $\rho(A)<1$. Thus, from the proof of Theorem~\ref{Theorem: main}, we have with probability of at least $1-\delta$,  
\begin{align*}
    \|\hat{S}_m-[A^m\sigma^2]_{\{i,j\}\in\mathcal{O}}\|_{\max}
    \le \frac{8\sigma^2(1+2\sqrt{\log(2N^2/\delta)})}{(1-\rho(A))^2\sqrt{T}} =: \bar{\epsilon}
\end{align*}
for $m=0,1,2,3$ and $T\ge 2(m+2)$. This requires $T\ge 6$, $T\ge 8$, $T\ge 10$, for $m=1,2,3$ respectively.

Now, we will show that the above condition implies with probability of at least $1-\delta$, $|\hat{\sigma}^2 - \sigma^2| \le \bar{\epsilon}$, $\|\hat{B} - B\|_{\max} \le \frac{4\bar{\epsilon}}{\sigma^2}$, $\|\widehat{CC^\top} - CC^\top\|_{\max} \le \frac{20N\bar{\epsilon}}{\sigma^2}$, and $\|\widehat{CEC^\top} - CEC^\top\|_{\max} \le \frac{140N^2\bar{\epsilon}}{\sigma^2}$.

First, since the trace is linear,
\[
|\hat{\sigma}^2 - \sigma^2|
= \left|\frac{1}{n_o}\Tr(\hat{S}_0) - \frac{1}{n_o}\Tr(\sigma^2I_{n_o}) \right|  =  \tfrac{1}{n_o}|\Tr(\hat{S}_0 - \sigma^2 I_{n_o})|
\le \tfrac{1}{n_o}\sum_i |\hat{S}_{0,ii}-\sigma^2|
\le \bar{\epsilon}.
\]
Second, note that \(\|B\|_2,\|C\|_2,\|E\|_2<1.\)
Using $\hat{B}=\hat{S}_1/\hat{\sigma}^2$ and applying triangular inequality to $ \|\hat{B}-B\|_{\max} = \|\frac{1}{\hat{\sigma}^2}(\hat{S}_1-B\sigma^2)  + B\frac{\sigma^2}{\hat{\sigma}^2} - B\|_{\max}$, we have  
    \begin{align*}
        \|\hat{B}-B\|_{\max}  \le \|\frac{1}{\hat{\sigma}^2}(\hat{S}_1-B\sigma^2)\|_{\max} + \|B(\frac{\sigma^2}{\hat{\sigma}^2}-1)\|_{\max} \le \frac{\bar{\epsilon}}{\hat{\sigma}^2} + \|B\|_{\max}\frac{|\sigma^2-\hat{\sigma}^2|}{\hat{\sigma}^2} \le \frac{2\bar{\epsilon}}{\hat{\sigma}^2}  \le \frac{4\bar{\epsilon}}{\sigma^2}.
    \end{align*}
Third, by triangular inequality 
    \begin{align*}
        \hat{\sigma}^2\|\widehat{CC^\top} - CC^\top\|_{\max} &\le \|(\widehat{CC^\top}{+}\hat{B}^2)\hat{\sigma}^2 - (CC^\top{+}B^2)\sigma^2 \|_{\max} + \|B^2\sigma^2 {-} \hat{B}^2\hat{\sigma}^2\|_{\max} + \|CC^\top(\sigma^2{-}\hat{\sigma}^2)\|_{\max},
    \end{align*}
    where each norm is bounded as: $\|(\widehat{CC^\top}+\hat{B}^2)\hat{\sigma}^2 - (CC^\top+B^2)\sigma^2 \|_{\max}\le \bar{\epsilon}$, 
    \begin{align*}
        \|B^2\sigma^2 - \hat{B}^2\hat{\sigma}^2\|_{\max} 
        &\le \|B^2\sigma^2  -B\hat{B}\sigma^2 \|_{\max} + \| B\hat{B}\sigma^2  - \hat{B}^2\sigma^2 \|_{\max} + \| \hat{B}^2\sigma^2 - \hat{B}^2\hat{\sigma}^2\|_{\max} \\
        &\le n_o\|B\|_{\max}\|B-\hat{B}\|_{\max}\sigma^2 + n_o\|B-\hat{B}\|_{\max}\|\hat{B}\|_{\max}\sigma^2 + \|\hat{B}^2\|_{\max} |\sigma^2-\hat{\sigma}^2| \le 9n_o\bar{\epsilon}, 
    \end{align*}
    $\|CC^\top(\sigma^2-\hat{\sigma}^2)\|_{\max} \le n_h\|C\|_{\max}^2|(\sigma^2-\hat{\sigma}^2)| \le n_h\bar{\epsilon}$. Thus, $\|\widehat{CC^\top} - CC^\top\|_{\max} \le \frac{(10n_o+n_h)\bar{\epsilon}}{\hat{\sigma}^2} \le \frac{20N\bar{\epsilon}}{\sigma^2}$.

Similarly, by triangular inequality and simplification, we have $\hat{\sigma}^2\|\widehat{CEC^\top} - CEC^\top\|_{\max} \le \bar{\epsilon} +  n_h^2\|C\|_{\max}^2\|E\|_{\max}|(\sigma^2-\hat{\sigma}^2)| + 9n_o\bar{\epsilon} + \|CC^\top B\sigma^2-\widehat{CC^\top}\hat{B}\hat{\sigma}^2\|_{\max} + \|BCC^\top\sigma^2 - \hat{B}\widehat{CC^\top}\hat{\sigma}^2\|_{\max}$, 
where $\|CC^\top B\sigma^2-\widehat{CC^\top}\hat{B}\hat{\sigma}^2\|_{\max} \le (8n_o n_h + 30n_o^2)\bar{\epsilon}$ and $\|BCC^\top\sigma^2 - \hat{B}\widehat{CC^\top}\hat{\sigma}^2\|_{\max} \le (8n_o n_h + 30n_o^2)\bar{\epsilon}$.
  
Therefore, 
    \begin{align*}
        \|\widehat{CEC^\top} - CEC^\top\|_{\max} \le \frac{(1+n_h^2+9n_o + 16n_o n_h + 60n_o^2)\bar{\epsilon}}{\hat{\sigma}^2}  \le \frac{2(60N^2+10N)\bar{\epsilon}}{\sigma^2} \le \frac{140N^2\bar{\epsilon}}{\sigma^2}.
    \end{align*}

Finally, by requiring $\bar{\epsilon}$ and $\frac{4\bar{\epsilon}}{\sigma^2}$ to be at most $\epsilon$, we arrive at the first stated lower bounds on $T$. Requiring $\frac{20N\bar{\epsilon}}{\sigma^2} \le \epsilon$ yields the second lower bound on $T$, and requiring $\frac{140N^2\bar{\epsilon}}{\sigma^2} \le \epsilon$ yields the third. As a result, under the displayed conditions on $T$, the block estimation errors are bounded by $\epsilon$ with probability of at least $1-\delta$.
\end{proof}

\section{Numerical Experiments}

We evaluate the performance of the proposed estimator on both fully observed and partially observed linear dynamical systems to empirically illustrate our theoretical complexity bounds. All experiments involve symmetrically coupled systems of dimension $N$, with adjacency matrix $A \in \mathbb{R}^{N\times N}$ and standard Gaussian noise driving the dynamics. For each setting, we compare our estimator against two common baselines: the standard least-squares (LS) estimator and the $\ell_1$-regularized least-squares (LS+L1) estimator.

\subsection{Fully Observed Systems}

\begin{figure}[htb!]
    \centering
    \includegraphics[width=0.95\linewidth]{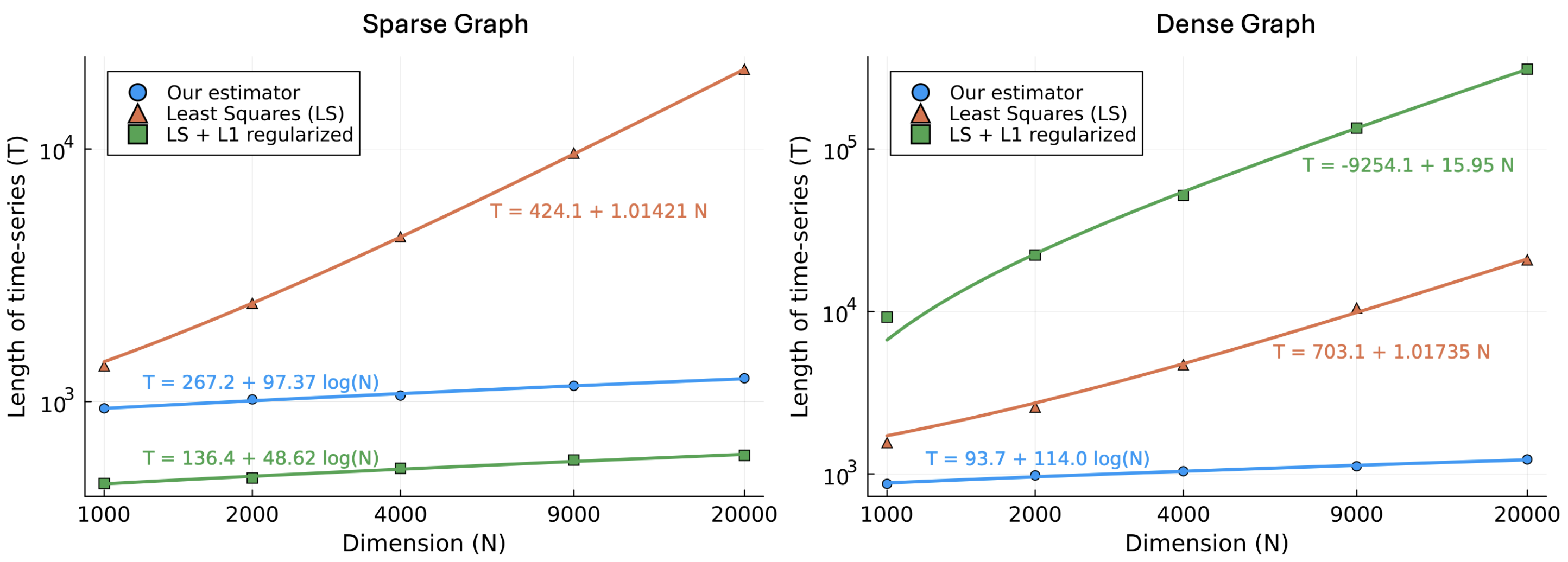}
    \caption{Sample complexity comparison for fully observed systems. 
    (Left) Sparse graph (degree 2). 
    (Right) Dense graph with structured first row/column. The curves represent the best fit of the empirical points. Our estimator and LS+L1 both achieve $\mathcal{O}(\log N)$ scaling in the sparse case, while only our estimator maintains logarithmic scaling in the dense regime.}
    \label{fig:sparse_dense}
\end{figure}

In the fully observed case, we study two classes of symmetric networks distinguished by the structure of the adjacency matrix $A$:\vspace{-0.25cm}
\begin{itemize}
    \item \textbf{Sparse Graph:} Each node (vertex) is connected to exactly two neighbors chosen uniformly at random. The resulting graph has a fixed degree of $2$, yielding a sparse symmetric adjacency matrix $A$ with spectral radius $\rho(A)<1$ after normalization (by a factor of 3).

    \item \textbf{Dense Star Graph:} The adjacency matrix $A$ follows a specific structured form where the first row and column are given by
    \[
    A_{1,:} = A_{:,1}^\top = \Big[\tfrac{1}{\sqrt{5}}, \tfrac{1}{\sqrt{2N}}, \dots, \tfrac{1}{\sqrt{2N}}\Big],
    \]
    and the remaining entries are set to 0. The construction enables us to produce a highly connected graph (where every node is connected to a central node) while preserving bounded spectral radius.
\end{itemize}

For both graphs, we simulate trajectories of varying lengths $T$ and dimensions $N \in [500, 20000]$, estimate $A$ using each tested method, and record the minimal trajectory length needed to achieve a fixed reconstruction accuracy with the error threshold of $0.25$. For each graph, 30 sets of samples is generated in order to test the prediction of our theory in the ``with high probability'' setting. The minimum number of samples that is sufficient to reconstruct the parameters to the threshold accuracy is recorded, and the maximum over 30 trials is reported in Figure~\ref{fig:sparse_dense}.

For the random sparse graph (left panel), both our estimator and the LS+L1 estimator display logarithmic scaling of the required trajectory length $T$ with respect to system dimension $N$, consistent with the theoretical prediction $T = \mathcal{O}(\log N)$. The unregularized LS estimator, however, fails to leverage the underlying sparsity and exhibits linear scaling in $N$. 
In contrast, for the random dense graph (right panel), only our estimator retains the favorable logarithmic dependence on $N$, while both LS+L1 and LS estimators require trajectories that grow linearly with system size. 
These results demonstrate that our method preserves the logarithmic sample complexity in learning both sparsely and densely connected systems, without requiring a problem-specific regularization.

\subsection{Partially Observed Systems}

\begin{figure}[htb!]
    \centering
    \includegraphics[width=0.95\linewidth]{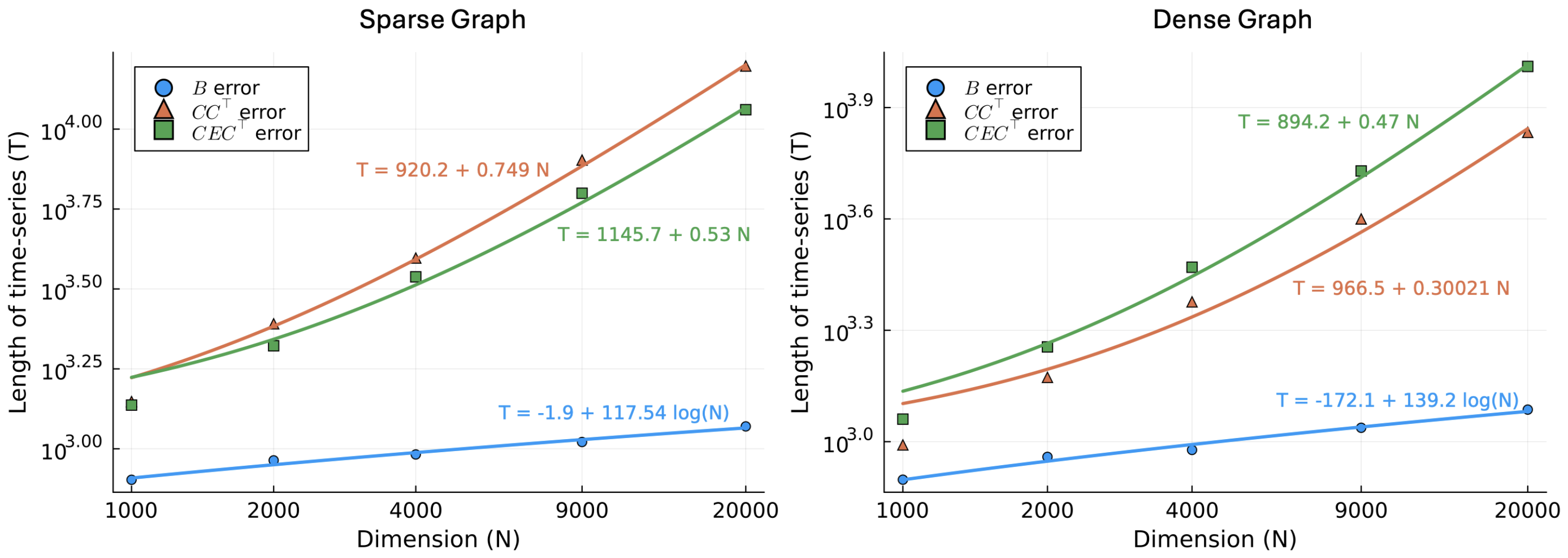}
    \caption{Performance under partial observability. The curves represent the best fit of the empirical points. The observable sub-matrix $B$ of the system matrix is recovered with $\mathcal{O}(\log N)$ complexity, while recovery of Markov parameters mixing other other sub-blocks involving hidden dimensions show an empirical scale of $\mathcal{O}(N \log N)$.}
    \label{fig:partial_obs}
\end{figure}

We next evaluate the estimator under partial observability, where only a subset of the system states is measured. In our experiments below, we choose $n_o = N/2$. The estimation performance in this regime is characterized by Corollary~\ref{corollary1}, which shows that the observed subsystem $B$ of the system matrix can be recovered using only $\mathcal{O}(\log N)$ samples, while accurate estimation of $CC^\top$ and $CEC^\top$ requires $\mathcal{O}(N^2\log N)$ and $\mathcal{O}(N^4\log N)$ samples, respectively.

In our numerical experiments, we compute the estimation errors of $B$, $CC^\top$, and $CEC^\top$, and the minimal trajectory length $T$ needed for these errors to fall below the error threshold of $0.25$. Each experiment is again repeated over 30 sets of samples, and the maximum of minimum $T$ needed to achive the required threshold error over 30 trials is recorded.

Figure~\ref{fig:partial_obs} presents the empirical scaling for accurate estimation of $B$, $CC^\top$, and $CEC^\top$, for two graph systems used and described earlier in the fully observed case. 
Across both sparse and dense systems, the minimal trajectory length $T$ required for accurate estimation of the observed block $B$ grows with logarithmic complexity in $N$, while trajectory length required for the unobserved terms $CC^\top$ and $CEC^\top$ empirically scales as $T = \mathcal{O}(N\log N)$. These numerical results show that in practice the scaling of $CC^\top$ and $CEC^\top$ recovery shows a better rate than the more conservative theoretical scaling predicted by Corollary~\ref{corollary1}.


\section{Conclusion}
In this work, we introduced a new estimator for learning parameters of stable and symmetrically coupled systems linear dynamical systems from finite data. Focusing on the maximum element-wise norm recovery which is relevant for applications such as structure discovery, our analysis showed that the observed part of the dynamic state matrix can be recovered using the length of the time series that scales only as $T = \mathcal{O}(\log N)$. Importantly, the estimator does not require regularization, and works for both sparse and dense dynamic state matrices. In future work, it would be interesting to explore versions of the new estimator that would work for marginally stable systems and non-symmetrically coupled matrices, maintaining efficient scaling rates.

\section*{Acknowledgements}
The authors acknowledge fruitful discussions with Melvyn Tyloo and Mateusz Wilinski, and thank Khanh Dang for providing the scientific data. This work has been supported by the U.S. Department of Energy/Office of Electricity Advanced Sensor and Data Analytics and Transmission Reliability and Operations (TRO) programs, by the Department of Energy/Office of Electricity Advanced Grid Modeling program, by the U.S. Department of Energy/Office of Science Advanced Scientific Computing Research program, and by the Information Science \& Technology Institute at Los Alamos National Laboratory.

\bibliographystyle{alpha}
\bibliography{bibfile}

\end{document}